\documentclass[a4paper]{article}

\usepackage{amsmath,amssymb,amsthm}
\usepackage{bbm}
\usepackage{lipsum}
\usepackage{color}
\usepackage{fullpage}
\usepackage{graphicx}

\usepackage{hyperref}

\usepackage{algorithm2e}
\usepackage{algorithmic}

\newtheorem{theorem}{Theorem}
\newtheorem{lemma}{Lemma}

\newcommand{\LEARNBOUND}{D}

\newcommand{\CSIZE}{f}

\newcommand{\RR}{{\mathbbm R}}

\newcommand{\E}{{\mathbbm E}}

\begin{document}

\title{Stochastic Contextual Bandits with Graph-based Contexts}

\author{
  Jittat Fakcharoenphol\thanks{Email: jittat@gmail.com. Department of Computer Engineering, Kasetsart University, Bangkok, Thailand, 10900.  Supported by the Thailand Research Fund, Grant RSA-6180074.}
  \and
  Chayutpong Prompak\thanks{Email: chay.promp@gmail.com. Department of Computer Engineering, Kasetsart University, Bangkok, Thailand, 10900.}
}

\maketitle



\begin{abstract}
We naturally generalize the on-line graph prediction problem to a
version of stochastic contextual bandit problems where contexts are
vertices in a graph and the structure of the graph provides
information on the similarity of contexts. More specifically, we are
given a graph $G=(V,E)$, whose vertex set $V$ represents contexts with
{\em unknown} vertex label $y$.  In our stochastic contextual bandit
setting, vertices with the same label share the same reward
distribution.  The standard notion of instance difficulties in graph
label prediction is the cutsize $\CSIZE$ defined to be the number of
edges whose end points having different labels.  For line graphs and
trees we present an algorithm with regret bound of
$\tilde{O}(T^{2/3}K^{1/3}\CSIZE^{1/3})$ where $K$ is the number of
arms.  Our algorithm relies on the optimal stochastic bandit algorithm
by Zimmert and Seldin~[AISTAT'19, JMLR'21].  When the best arm
outperforms the other arms, the regret improves to
$\tilde{O}(\sqrt{KT\cdot \CSIZE})$.  The regret bound in the later
case is comparable to other optimal contextual bandit results in more
general cases, but our algorithm is easy to analyze, runs very
efficiently, and does not require an i.i.d. assumption on the input
context sequence.  The algorithm also works with general graphs using
a standard random spanning tree reduction.
\end{abstract}


\section{Introduction}

Side information for many learning problems can be represented as a
graph.  For example, since a ``type'' of a social network user is
typically highly correlated with the user's friends, social network
service providers, with full knowledge of the social network graph,
can use the graph structure to help making decision on which
advertisement to show to the user.  It can also use the graph
structure to infer various properties of its users, e.g., their life
styles, social statuses, and marketing values.

User type prediction can be modeled as follows.  Given an $n$-vertex
graph $G=(V,E)$, a label set ${\mathcal L}$, and an unknown vertex
labels $y:V\rightarrow {\mathcal L}$ representing certain vertex
classification, the {\em on-line label prediction problem}, studied in
Herbster, Pontil, and Wainer~\cite{HerbsterPW05-icml-learning-graphs};
Herbster and Pontil~\cite{HerbsterP06-nips-perceptron}; Herbster,
Lever, and Pontil~\cite{HerbsterLP08-large-diameter}; and
Cesa-Bianchi, Gentile, and
Vitale~\cite{Cesa-BianchiGV09-optimal-tree}, works in rounds, for each
round, a vertex is queried, and the learning algorithm predicts the
label of that particular vertex.  The available graph structure
provides clues to the learner by promising that adjacent vertices
usually have the same label.  The measure of complexity of the problem
instance is the {\em cutsize} $\CSIZE=\CSIZE_G(y)$ defined as the
number of edges $(u,v)\in E$ such that $y(u)\neq y(v)$.  We note that
while the notion of cutsize is very natural when the graph is a tree,
in a general graph the cutsize can be large, e.g., for dense graph the
cutsize can be $\Omega(n^2)$.  Previous results on graph prediction
guarantee mistake bounds that depend linearly on $\CSIZE$ and
logarithmically on the size of the graph $n$.

In most case, we have no access to vertex real labels and more
importantly, our goal is to make decisions instead of figuring out the
labels.  Also, the problem itself may have elements of uncertainty.
For example, the same user may give different responses when asked
many times.

With this motivation, we generalize the label prediction to a version
of stochastic contextual bandit problems where contexts are vertices
in a graph and the structure of the graph provides information on the
similarity of contexts.  More specifically, the unknown label $y$
defines the context group and vertices from the same context group
share the same reward behavior (as modeled by a bandit problem).

We are interested in the case where the reward behavior is represented
as a stochastic multi-armed bandit problem with $K$ arms.  Given graph
$G=(V,E)$ (with unknown vertex label $y$), the algorithm proceeds in
$T$ rounds and, for each round $t$, receives a context $u\in V$.  The
algorithm has to pick an arm $I_t$ to maximize the total rewards.  The
rewards stochastically depend on the unknown vertex label $y(u)$ and
the arm $I$.  Unlike the label prediction problem, a single vertex $u$
can be queried multiple times.  We only consider the case of $[0,1]$
rewards.

Essentially we deal with line graphs.  Given the algorithm for line
graphs, a simple reduction by Herbster, Lever, and
Pontil~\cite{HerbsterLP08-large-diameter} allows us to work with a
tree with $\CSIZE$ cut edges using a line graph with $2\cdot\CSIZE$
cut edges.  For general graphs, we work with their random spanning
tree representation as in, e.g.,~\cite{Cesa-BianchiGV09-optimal-tree,
  HerbsterLP08-large-diameter, Cesa-BianchiGVZ10-icml-random-trees}.
We state relevant results in Section~\ref{sect:algo-gen-graphs}.

Our main result on line graphs (and trees) is the following.

\begin{theorem}
  Consider a line graph or a tree with $n$ vertices and $\CSIZE$ cut edges.
  There is an efficient bandit algorithm that has a regret bound of
  $\tilde{O}(T^{2/3}\cdot K^{1/3}\cdot \CSIZE^{1/3})$, where $T$ is the number of
  rounds and $K$ is the number of arms.
  \label{thm:main1}
\end{theorem}

Our algorithm is a very simple divide-and-conquer algorithm that
hierarchically decompose the contexts into many levels.  Originally,
our algorithm relies heavily on the best-arm identification algorithm,
called Successive Rejects, of Audibert, Bubeck, and
Munos~\cite{AudibertBM10} with regret bound of
$\tilde{O}(T^{3/4}\cdot K^{1/4}\cdot \CSIZE^{1/4})$.  The simpler and
better version using an optimal bandit subroutine by Zimmert and
Seldin~\cite{jmlr21-zimmer-seldin-tsallis} presented here has been
suggested by Anonymous Reviewer \# 2 from ALT'21.

It still requires the prior knowledge of the total number of rounds
$T$.  We note that there are dependencies between $T$ and $K$, i.e.,
$K\leq T$ so that the bound does not implies an upper bound better
than $\Omega(\sqrt{KT})$.  The above bound is distribution
independent, i.e., it works for any reward distribution.

However, in an easy case, where the best arm outperforms the other
arms, the Tsallis-INF algorithm has a better regret bound and, in
turn, we also have an improved bound of $\tilde{O}(\sqrt{KT\cdot
  \CSIZE})$.

One notable property of the algorithm is that it does not aim to
identify the cut edges, so there is no need for a requirement that
different context groups should behave differently.

Also, the parameter $\CSIZE$ is actually the ``observable'' cutsize,
i.e., if some vertex $u$ does not appear as contexts we are free to
assume its label $y(u)$ to be anything and the observable cutsize is
the minimum cutsize obtainable after some label renaming.

As our work is very simple (after the suggestion from an anonymous
reviwer), we defer the discussion and comparison with related work to
Section~\ref{sect:related}.  Section~\ref{sect:prelim} gives formal
settings.  The algorithm for line graphs and its analysis is presented
in Section~\ref{sect:algo-line}.  We briefly discuss how to extend the
line-graph result to trees and general graphs in
Section~\ref{sect:algo-gen-graphs}.

\section{Preliminaries}
\label{sect:prelim}

\subsection{Settings}

We are given a context set $S = \{c_1,c_2,c_3,...,c_N\}$ and arms
$1,2,\ldots,K$.  At each time step $t=1,2,\ldots,$ the player receives
a context $c\in S$ and has to play an arm $I_t \in \{1,\ldots,K\}$
then incurs reward $X_{I_t,t} \in [0,1]$.  The goal is to maximize
reward.  For contextual stochastic bandits strategy, the reward
$X_{I_t,t}$ is drawn independently from the past from an unknown
distribution $\nu_{I_t}^{c}$ that depends only on the context $c$ and
the arm $I_t$.  In this work we consider losses, defined as
$\ell_{I_t,t} = 1-X_{I_t,t}$ rather than rewards.

The learning problem proceeds in $T$ rounds.  For each round, the
algorithm is to map contexts to arms and minimizes the loss. The
total losses after $T$ rounds is $\sum_{t=1}^T \ell_{I_t,t}$.

For context $c$ and arm $i$, let $\mu_i^c$ be the mean reward from
distribution $\nu_i^c$.  Let $i^*_c$ be the best arm for context $c$,
i.e., $i^*_c = \arg_i\max_{i=1}^K\mu_i^c$.  Thus for context $c$, the
minimum expected loss, denoted by $loss^*_c$, is
\[
1 - \mu_{i^*_c}^c.
\]
Thus, given a sequence of contexts $s_1,s_2,\ldots,s_T$, the minimum
expected loss, denoted as $loss^*$, is
\[
loss^* = \sum_{t=1}^T loss^*_{s_t}.
\]

We are also given a graph $G=(S,E)$ with an unknown vertex label
$y:S\rightarrow L$ for some label set $L$.  We sometimes refer to a
context in $S$ as a vertex.  Every context with the same label share
the same reward distribution, i.e., for every pair of context $c,c'
\in S$ such that $y(c)=y(c')$, for every arm $i$, the distribution
$v_i^{c}=v_i^{c'}$.  Note that we have no access to the labels $y$.
The labels $y$ partition contexts into {\em context groups}. When
context $c$ and $c'$ share the same label, we say that $c$ and $c'$
are {\em from the same group}.  Contexts from the same group share the
same best arm.

Given the graph $G=(S,E)$ and vertex label $y$, an edge $(u,v)\in E$
is a {\em cut edge} if $y(u)\neq y(v)$.  Let $\CSIZE$ be the number of
cut edges.

In this work, we mostly work with line graphs.  The result extends to
trees and general graphs using standard techniques (see
Section~\ref{sect:algo-gen-graphs}).

\subsection{The multi-armed bandit problem}
\label{sect:prelim-mab}

Since the underlying problem for a specific context $c$ is the
stochastic version of the multi-armed bandit, we start by reviewing
the problem and algorithms for solving it in this section.  In this
problem, there are $K$ arms; each arm $i$ corresponds to an unknown
probability distribution $\nu_i$.  For each round $t=1,\ldots,T$, the
algorithm picks an arm $I_t\in\{1,\ldots,K\}$ and receives rewards
$X_{I_t,t}$ drawn from $\nu_i$.  Let $\mu_i$ be the mean of arm $i$;
the best arm $i^*$ is $\arg_i\max_i\mu_i$.  Let $\mu^*=\mu_{i^*}$.
The pseudo-regret for the algorithm is
\[
T\mu^*-\E\left[\sum_{t=1}^T\mu_{I_t}\right].
\]
Since the learning becomes harder when the expected rewards of other
arms are close to $\mu_i$, we define $\Delta_i=\mu^*-\mu_i$ to
represent how close arm $i$ to the best arm.  As
in~\cite{jmlr21-zimmer-seldin-tsallis}, we assume that there is a
unique best arm $i^*$ where $\Delta_{i^*}=0$ and $\Delta_i>0$ for all
$i\neq i^*$.  We let $\Delta_{\mathrm{min}}=\min_{\Delta_i>0}\Delta_i$.

Our main subroutine is the Tsallis-INF algorithm
by~\cite{jmlr21-zimmer-seldin-tsallis}, which is an algorithm based
on online mirror descent with Tsallis entropy regularization with
power $\alpha=1/2$.  The algorithm uses an unbiased estimator for the
loss $\hat{\ell}_{i,t}$ for arm $i$ at time $t$, such that
$\E_{I_t}[\hat{\ell}_{i,t}]=\ell_{i,t}$ internally.  While
\cite{jmlr21-zimmer-seldin-tsallis} proposed two alternatives for
these estimators, we focus on the one based on importance-weighted
sampling referred to as (IW) estimators and only state their result
for this type of estimators below.

\begin{theorem}[\cite{jmlr21-zimmer-seldin-tsallis}]
  The pseudo-regret of Tsallis-INF with $\alpha=1/2$ using (IW)
  estimators is at most
  \[
  4\sqrt{KT}+1,
  \]
  where $K$ is the number of arms and $T$ is the number of rounds.
  Also if $\Delta_{\min}$ is a constant, the pseudo-regret can be
  bounded by
  \[
  O(K\log T).
  \]
  \label{thm:tsallis-inf}
\end{theorem}

We refer to the case where $\Delta_{\min}=\Omega(1)$ as an easy case.
This case is discussed in Section~\ref{sect:easy}.


\section{An algorithm for line graphs and its analysis}
\label{sect:algo-line}

We start by describing a divide-and-conquer algorithm\footnote{ We
remark again that, originally, the algorithm is more complex than the
one presented here based on Tsallis-INF as suggested by an anonymous
referee from ALT'21.  We really appreciate the suggestion.  } for line
graphs, where, for $1\leq i<n$, vertex (context) $c_i$ is adjacent to
vertex $c_{i+1}$.  For simplicity, assume that $n$ is a power of two,
i.e., let $n=2^L$.

This algorithm hierarchically decomposes the context set into $O(\log
n)$ levels.  For each level, we partition the context set into a set
of the same size, and treat contexts in the same set as if they are
from the same unknown context group.  At the top level, we have only
two context sets, each of size $n/2$.  As the level decreases, the
size of each context set decreases exponentially.  We start by running
bandit subroutines on the highest level.  When an active subroutine
for context set $C'$ runs for a number of rounds, we terminate it,
pretend to split the context $C'$ by going down a level, and start
running two new bandit subroutines.

We now formally describe the algorithm.  Recall that $n=2^L$.  We
maintain stochastic bandit subroutines, each executes the Tsallis-INF
algorithm, in $L$ levels.  For level $p$, where $1\leq p\leq L$, there
are $n/2^{L-p}$ bandit subroutines, denoted by
$B_p(1),\ldots,B_p(n/2^{L-p})$.  Bandit subroutine $B_p(j)$,
essentially, deals with contexts
\[
\{c_{2^{(L-p)(j-1)}},\ldots,c_{2^{(L-p)j}}\}
\]
as if they are from the
same group.  We also say that bandit subroutine $B_p(j)$ is {\em
  responsible} for contexts
\[
\{c_{2^{(L-p)(j-1)}},\ldots,c_{2^{(L-p)j}}\}.
\]

For level $1\leq p<L$, for subroutine $B_p(j)$, we call $B_{p+1}(j_l)$
and $B_{p+1}(j_r)$, where $j_l=2j-1$ and $j_r=2j$, {\em children} of
$B_p(j)$.  Note that they are responsible for contexts
$\{c_{2^{(L-p)(j-1)}},\ldots,c_{2^{(L-p)j}/2}\}$ and
$\{c_{2^{(L-p)j}/2+1},\ldots,c_{2^{(L-p)j}}\}$, respectively.  (See
figure \ref{fig:subroutines}).

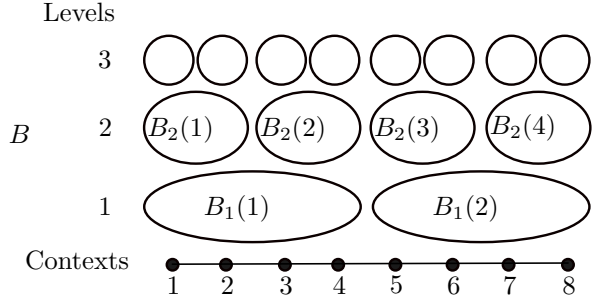
\begin{figure}
  \centering
  \def\svgwidth{3in}

  \begingroup%
  \makeatletter%
  \providecommand\color[2][]{%
    \errmessage{(Inkscape) Color is used for the text in Inkscape, but the package 'color.sty' is not loaded}%
    \renewcommand\color[2][]{}%
  }%
  \providecommand\transparent[1]{%
    \errmessage{(Inkscape) Transparency is used (non-zero) for the text in Inkscape, but the package 'transparent.sty' is not loaded}%
    \renewcommand\transparent[1]{}%
  }%
  \providecommand\rotatebox[2]{#2}%
  \newcommand*\fsize{\dimexpr\f@size pt\relax}%
  \newcommand*\lineheight[1]{\fontsize{\fsize}{#1\fsize}\selectfont}%
  \ifx\svgwidth\undefined%
    \setlength{\unitlength}{268.94231901bp}%
    \ifx\svgscale\undefined%
      \relax%
    \else%
      \setlength{\unitlength}{\unitlength * \real{\svgscale}}%
    \fi%
  \else%
    \setlength{\unitlength}{\svgwidth}%
  \fi%
  \global\let\svgwidth\undefined%
  \global\let\svgscale\undefined%
  \makeatother%
  \begin{picture}(1,0.50329222)%
    \lineheight{1}%
    \setlength\tabcolsep{0pt}%
    \put(0,0){\includegraphics[width=\unitlength,page=1]{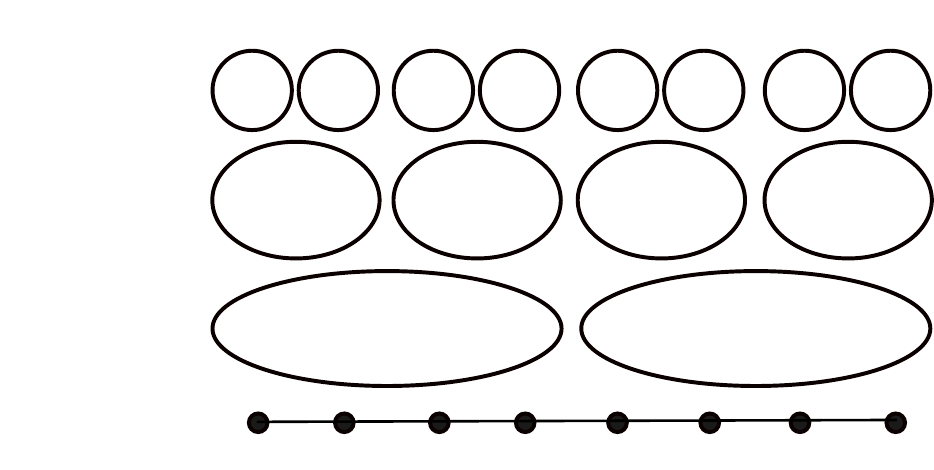}}%
    \put(0.0218773,0.04319166){\color[rgb]{0,0,0}\makebox(0,0)[lt]{\lineheight{1.25}\smash{\begin{tabular}[t]{l}Contexts\end{tabular}}}}%
    \put(0.1486307,0.13792125){\color[rgb]{0,0,0}\makebox(0,0)[lt]{\lineheight{1.25}\smash{\begin{tabular}[t]{l}1\end{tabular}}}}%
    \put(0.1489212,0.27520124){\color[rgb]{0,0,0}\makebox(0,0)[lt]{\lineheight{1.25}\smash{\begin{tabular}[t]{l}2\end{tabular}}}}%
    \put(0.14817681,0.39294764){\color[rgb]{0,0,0}\makebox(0,0)[lt]{\lineheight{1.25}\smash{\begin{tabular}[t]{l}3\end{tabular}}}}%
    \put(-0.39229362,1.52820172){\color[rgb]{0,0,0}\makebox(0,0)[lt]{\begin{minipage}{2.13058086\unitlength}\raggedright \end{minipage}}}%
    \put(0.05318665,0.47504209){\color[rgb]{0,0,0}\makebox(0,0)[lt]{\lineheight{1.25}\smash{\begin{tabular}[t]{l}Levels\end{tabular}}}}%
    \put(0.08200531,1.03131713){\color[rgb]{0,0,0}\makebox(0,0)[lt]{\begin{minipage}{1.50571093\unitlength}\raggedright \end{minipage}}}%
    \put(0.33054696,0.13761259){\color[rgb]{0,0,0}\makebox(0,0)[lt]{\lineheight{1.25}\smash{\begin{tabular}[t]{l}$B_1(1)$\end{tabular}}}}%
    \put(0.72545101,0.13761259){\color[rgb]{0,0,0}\makebox(0,0)[lt]{\lineheight{1.25}\smash{\begin{tabular}[t]{l}$B_1(2)$\end{tabular}}}}%
    \put(0.23301699,0.2751377){\color[rgb]{0,0,0}\makebox(0,0)[lt]{\lineheight{1.25}\smash{\begin{tabular}[t]{l}$B_2(1)$\end{tabular}}}}%
    \put(0.42703886,0.2751377){\color[rgb]{0,0,0}\makebox(0,0)[lt]{\lineheight{1.25}\smash{\begin{tabular}[t]{l}$B_2(2)$\end{tabular}}}}%
    \put(0.62434309,0.2751377){\color[rgb]{0,0,0}\makebox(0,0)[lt]{\lineheight{1.25}\smash{\begin{tabular}[t]{l}$B_2(3)$\end{tabular}}}}%
    \put(0.82440779,0.2773436){\color[rgb]{0,0,0}\makebox(0,0)[lt]{\lineheight{1.25}\smash{\begin{tabular}[t]{l}$B_2(4)$\end{tabular}}}}%
    \put(-0.00617291,0.26189317){\color[rgb]{0,0,0}\makebox(0,0)[lt]{\lineheight{1.25}\smash{\begin{tabular}[t]{l}$B$\end{tabular}}}}%
    \put(0.26673118,0.0004067){\color[rgb]{0,0,0}\makebox(0,0)[lt]{\lineheight{1.25}\smash{\begin{tabular}[t]{l}1\end{tabular}}}}%
    \put(0.35964921,0.00021061){\color[rgb]{0,0,0}\makebox(0,0)[lt]{\lineheight{1.25}\smash{\begin{tabular}[t]{l}2\end{tabular}}}}%
    \put(0.4610726,0.00042121){\color[rgb]{0,0,0}\makebox(0,0)[lt]{\lineheight{1.25}\smash{\begin{tabular}[t]{l}3\end{tabular}}}}%
    \put(0.55308225,0.0004067){\color[rgb]{0,0,0}\makebox(0,0)[lt]{\lineheight{1.25}\smash{\begin{tabular}[t]{l}4\end{tabular}}}}%
    \put(0.65195142,0.0006173){\color[rgb]{0,0,0}\makebox(0,0)[lt]{\lineheight{1.25}\smash{\begin{tabular}[t]{l}5\end{tabular}}}}%
    \put(0.75047429,0.00042121){\color[rgb]{0,0,0}\makebox(0,0)[lt]{\lineheight{1.25}\smash{\begin{tabular}[t]{l}6\end{tabular}}}}%
    \put(0.84727944,0.0004067){\color[rgb]{0,0,0}\makebox(0,0)[lt]{\lineheight{1.25}\smash{\begin{tabular}[t]{l}7\end{tabular}}}}%
    \put(0.94977909,0.00042121){\color[rgb]{0,0,0}\makebox(0,0)[lt]{\lineheight{1.25}\smash{\begin{tabular}[t]{l}8\end{tabular}}}}%
  \end{picture}%
\endgroup%

  \caption{Subroutines and contexts.  $B_1(1)$ is {\em responsible} for context set \{1,2,3,4\}. $B_2(1)$ and $B_2(2)$ are children of $B_1(1)$.}
  \label{fig:subroutines}
\end{figure}

There is a parameters $\LEARNBOUND$ that defines the number of rounds
a bandit subroutine handles requests on its responsible contexts.  Each
subroutine $B_p(j)$ not on level $L$ would runs the Tsallis-INF
algorithm for $D$ rounds, after that it deactivates itself and
activates its children $B_{p+1}(j_l)$ and $B_{p+1}(j_r)$.


\subsection{Analysis of pseudo-regret for line graphs}
\label{sect:analysis}
  
We essentially analyze the case for line graphs.  The analysis for the
general cases (for trees and general graphs) follows from standard
reductions (see Section~\ref{sect:algo-gen-graphs}).

For each bandit subroutine $B_p(j)$, we say that the subroutine is in
{\em good} situation if its responsible contexts are from the same
context group. Otherwise, we say that it is in a {\em bad} situation.
(See figure \ref{fig:situation}).

\begin{figure}
  \centering
  \def\svgwidth{3in}

\begingroup%
  \makeatletter%
  \providecommand\color[2][]{%
    \errmessage{(Inkscape) Color is used for the text in Inkscape, but the package 'color.sty' is not loaded}%
    \renewcommand\color[2][]{}%
  }%
  \providecommand\transparent[1]{%
    \errmessage{(Inkscape) Transparency is used (non-zero) for the text in Inkscape, but the package 'transparent.sty' is not loaded}%
    \renewcommand\transparent[1]{}%
  }%
  \providecommand\rotatebox[2]{#2}%
  \newcommand*\fsize{\dimexpr\f@size pt\relax}%
  \newcommand*\lineheight[1]{\fontsize{\fsize}{#1\fsize}\selectfont}%
  \ifx\svgwidth\undefined%
    \setlength{\unitlength}{268.44771318bp}%
    \ifx\svgscale\undefined%
      \relax%
    \else%
      \setlength{\unitlength}{\unitlength * \real{\svgscale}}%
    \fi%
  \else%
    \setlength{\unitlength}{\svgwidth}%
  \fi%
  \global\let\svgwidth\undefined%
  \global\let\svgscale\undefined%
  \makeatother%
  \begin{picture}(1,0.59468884)%
    \lineheight{1}%
    \setlength\tabcolsep{0pt}%
    \put(0,0){\includegraphics[width=\unitlength,page=1]{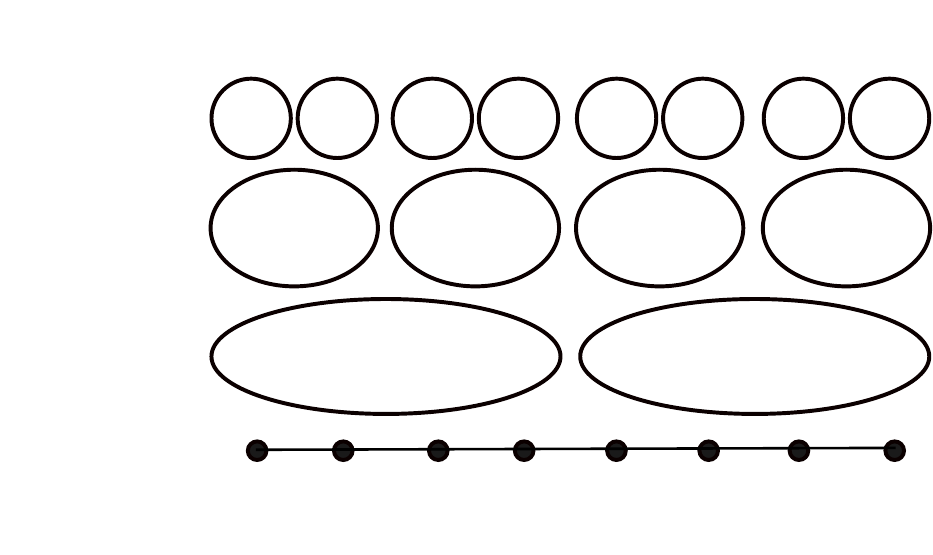}}%
    \put(0.02078988,0.10377121){\color[rgb]{0,0,0}\makebox(0,0)[lt]{\lineheight{1.25}\smash{\begin{tabular}[t]{l}Contexts\end{tabular}}}}%
    \put(0.14777683,0.19867534){\color[rgb]{0,0,0}\makebox(0,0)[lt]{\lineheight{1.25}\smash{\begin{tabular}[t]{l}1\end{tabular}}}}%
    \put(0.14806786,0.33620827){\color[rgb]{0,0,0}\makebox(0,0)[lt]{\lineheight{1.25}\smash{\begin{tabular}[t]{l}2\end{tabular}}}}%
    \put(0.1473221,0.4541716){\color[rgb]{0,0,0}\makebox(0,0)[lt]{\lineheight{1.25}\smash{\begin{tabular}[t]{l}3\end{tabular}}}}%
    \put(-0.39414413,1.59151735){\color[rgb]{0,0,0}\makebox(0,0)[lt]{\begin{minipage}{2.13450638\unitlength}\raggedright \end{minipage}}}%
    \put(0.05215692,0.53641731){\color[rgb]{0,0,0}\makebox(0,0)[lt]{\lineheight{1.25}\smash{\begin{tabular}[t]{l}Levels\end{tabular}}}}%
    \put(0.08102868,1.09371727){\color[rgb]{0,0,0}\makebox(0,0)[lt]{\begin{minipage}{1.50848515\unitlength}\raggedright \end{minipage}}}%
    \put(0.37793824,0.19836611){\color[rgb]{0,0,0}\makebox(0,0)[lt]{\lineheight{1.25}\smash{\begin{tabular}[t]{l}Bad\end{tabular}}}}%
    \put(0.77356993,0.19836611){\color[rgb]{0,0,0}\makebox(0,0)[lt]{\lineheight{1.25}\smash{\begin{tabular}[t]{l}Bad\end{tabular}}}}%
    \put(0.26727259,0.3361446){\color[rgb]{0,0,0}\makebox(0,0)[lt]{\lineheight{1.25}\smash{\begin{tabular}[t]{l}Good\end{tabular}}}}%
    \put(0.46165193,0.3361446){\color[rgb]{0,0,0}\makebox(0,0)[lt]{\lineheight{1.25}\smash{\begin{tabular}[t]{l}Good\end{tabular}}}}%
    \put(0.67156097,0.3361446){\color[rgb]{0,0,0}\makebox(0,0)[lt]{\lineheight{1.25}\smash{\begin{tabular}[t]{l}Bad\end{tabular}}}}%
    \put(0.85688823,0.33835458){\color[rgb]{0,0,0}\makebox(0,0)[lt]{\lineheight{1.25}\smash{\begin{tabular}[t]{l}Good\end{tabular}}}}%
    \put(-0.007312,0.32287567){\color[rgb]{0,0,0}\makebox(0,0)[lt]{\lineheight{1.25}\smash{\begin{tabular}[t]{l}B\end{tabular}}}}%
    \put(0.2660949,0.06090742){\color[rgb]{0,0,0}\makebox(0,0)[lt]{\lineheight{1.25}\smash{\begin{tabular}[t]{l}1\end{tabular}}}}%
    \put(0.35918413,0.06071097){\color[rgb]{0,0,0}\makebox(0,0)[lt]{\lineheight{1.25}\smash{\begin{tabular}[t]{l}2\end{tabular}}}}%
    \put(0.46079439,0.06092196){\color[rgb]{0,0,0}\makebox(0,0)[lt]{\lineheight{1.25}\smash{\begin{tabular}[t]{l}3\end{tabular}}}}%
    \put(0.55297356,0.06090742){\color[rgb]{0,0,0}\makebox(0,0)[lt]{\lineheight{1.25}\smash{\begin{tabular}[t]{l}4\end{tabular}}}}%
    \put(0.6520249,0.06111841){\color[rgb]{0,0,0}\makebox(0,0)[lt]{\lineheight{1.25}\smash{\begin{tabular}[t]{l}5\end{tabular}}}}%
    \put(0.75072929,0.06092196){\color[rgb]{0,0,0}\makebox(0,0)[lt]{\lineheight{1.25}\smash{\begin{tabular}[t]{l}6\end{tabular}}}}%
    \put(0.8477128,0.06090742){\color[rgb]{0,0,0}\makebox(0,0)[lt]{\lineheight{1.25}\smash{\begin{tabular}[t]{l}7\end{tabular}}}}%
    \put(0.9504013,0.06092196){\color[rgb]{0,0,0}\makebox(0,0)[lt]{\lineheight{1.25}\smash{\begin{tabular}[t]{l}8\end{tabular}}}}%
    \put(0.64754711,0.4541716){\color[rgb]{0,0,0}\makebox(0,0)[lt]{\lineheight{1.25}\smash{\begin{tabular}[t]{l}G\end{tabular}}}}%
    \put(0.73997429,0.4541716){\color[rgb]{0,0,0}\makebox(0,0)[lt]{\lineheight{1.25}\smash{\begin{tabular}[t]{l}G\end{tabular}}}}%
    \put(0.27746011,0.02964715){\color[rgb]{0,0,0}\makebox(0,0)[lt]{\lineheight{1.25}\smash{\begin{tabular}[t]{l}group 1\end{tabular}}}}%
    \put(0.51970029,0.0294998){\color[rgb]{0,0,0}\makebox(0,0)[lt]{\lineheight{1.25}\smash{\begin{tabular}[t]{l}group 2\end{tabular}}}}%
    \put(0.8138518,0.0294998){\color[rgb]{0,0,0}\makebox(0,0)[lt]{\lineheight{1.25}\smash{\begin{tabular}[t]{l}group 3\end{tabular}}}}%
    \put(0,0){\includegraphics[width=\unitlength,page=2]{b_good.pdf}}%
  \end{picture}%
\endgroup%

  \caption{Subroutines in good and bad situations.}
  \label{fig:situation}
\end{figure}

Let $\mathcal C$ be the set of every context set that some bandit
subroutine is responsible for.

We first bound the number of bandit subroutines in bad situations.  A
subroutine $B_p(j)$ is in a bad situation when the subgraph of $G$
corresponding to the contexts that $B_p(j)$ is responsible for
contains a cut edge.  Recall that our decomposition contains $O(\log
n)$ levels; thus, a single cut edge may appear in $O(\log n)$ context
sets in $\mathcal C$ and contributes to $O(\log n)$ subroutines in bad
situations.  Since there are $f$ cut edges, the number of subroutines
in bad situations is $O(f\cdot\log n)$, yielding the following lemma.

\begin{lemma}
  The number of subroutines in bad situations is at most $O(f\cdot\log
  n)$.
  \label{lemma:bad-sub}
\end{lemma}

We are ready to prove the main theorem (which is a restatement of
Theorem~\ref{thm:main1}).

\begin{theorem}
  If the cutsize of a line graph is $f$, the pseudo-regret is at most
  $\tilde{O}(T^{2/3}K^{1/3}f^{1/3})$.
  \label{thm:main-regret}
\end{theorem}
\begin{proof}
  Since only two new subroutines are activated after their parent
  deactivates and each subroutine deactivates after handling
  $\LEARNBOUND$ requests, the number of subroutines activated is at
  most $2T/D$.

  From Lemma~\ref{lemma:bad-sub}, the number of subroutines in bad
  situations is at most $O(f\cdot\log n)$.  Each of these subroutines
  can make at most $\LEARNBOUND$ regret, a total of
  $O(\LEARNBOUND\cdot f\cdot\log n)$.

  From Theorem~\ref{thm:tsallis-inf}, each subroutine in a good
  situation contributes to at most
  \[
  4\sqrt{K\LEARNBOUND}+1,
  \]
  since it runs for at most $\LEARNBOUND$ rounds.

  Combining both bounds, we have that the regret incurred by the
  algorithm is at most
  \[
  O(\LEARNBOUND\cdot f\cdot\log n)+
  2T/\LEARNBOUND\cdot O(\sqrt{K\LEARNBOUND})
  =O(\LEARNBOUND\cdot f\log n + T\sqrt{K/\LEARNBOUND}).
  \]

  Choosing $D=\tilde{\Theta}(T^{2/3}K^{1/3}/f^{2/3})$, we have that
  the pseudo-regret is at most
  \[
  \tilde{O}(T^{2/3}K^{1/3}f^{1/3}),
  \]
  as claimed.
\end{proof}

\subsection{When the best arm is easy to identify}
\label{sect:easy}

In this section, we assume that $\Delta_{\min}$ is a constant.  In
this case, the second part of Theorem~\ref{thm:tsallis-inf} ensures
that the pseudo-regret for subroutines in good situations after
$\LEARNBOUND$ rounds is at most $O(K\log\LEARNBOUND)$.

Using the same analysis as in Theorem~\ref{thm:main-regret}, we have
that the pseudo-regret can be bounded by
\[
O(\LEARNBOUND\cdot f\log n + (T/D)\cdot K\log D).
\]
Choosing $\LEARNBOUND=\Omega(\sqrt{TK/f})$, the pseudo-regret becomes
\[
\tilde{O}(\sqrt{KT\cdot f}).
\]


\section{General graphs}
\label{sect:algo-gen-graphs} 

\subsection{Reduction from trees to line graphs}

As in the graph label prediction problem, Herbster, Lever, and
Pontil~\cite{HerbsterLP08-large-diameter} show that we can reduce the
problem when the graph is a tree to the case when the graph is a line
graph (called a spine) while paying a factor of 2 for the cutsize.
The spine graph can be constructed using depth-first search on the
tree.  We state the following lemma from~Section 4
in~\cite{HerbsterLP08-large-diameter}.

\begin{lemma}[\cite{HerbsterLP08-large-diameter}]
  Given a tree $T=(V,E)$ with vertex label $y$, there exists a path
  graph $G=(V',E')$, a mapping $g:V\rightarrow V'$, and a natural
  extension of vertex label $y'$ to $V'$ such that the cutsize of $G$
  under $y'$ is at most twice the cutsize of $T$ under $y$.
\end{lemma}

\subsection{Random spanning tree method}

To deal with general graphs, we simply use random spanning tree method
as described in Cesa-Bianchi, Gentile, and
Vitale~\cite{Cesa-BianchiGV09-optimal-tree}; Herbster, Lever, and
Pontil~\cite{HerbsterLP08-large-diameter}; and Cesa-Bianchi, Gentile,
Vitale, and Zappella~\cite{Cesa-BianchiGVZ10-icml-random-trees}.  The
cutsize parameter of the learning problem in this case becomes the
expected cutsize of a random spanning tree.


\section{Related work}
\label{sect:related}

The work present in this paper can be compared to results from
standard contextual bandit that deals with large number of policies.
There are also a number of results dealing directly with a set of
related bandit problems that either can be clustered in an online
fashion or can be partitioned into equivalence classes.

\subsection{Comparison to other contextual bandit algorithms}

When considering bandit performance, it is useful to clearly specify
the set for which its performance is compared against.  In the
contextual bandit settings, for a given context set $\mathcal S$, a
policy $g:{\mathcal S}\rightarrow\{1,\ldots,K\}$ is a mapping from the
context set to the set of arms.  Let $\Pi$ be the set of possible
policies and $N=|\Pi|$.  We would like the regret bound to be
sublinear in $T$ and logarithmic in $N$ because the set of policies
can be very large.  Note that in the context of graph prediction, if
the cutsize is $\CSIZE$, the number of possible policies $N$ can be as
large as ${n-1 \choose\CSIZE}=\Omega(n^\CSIZE)$.  Under this settings,
the contextual bandit algorithms in the adversarial settings such as
{\tt Exp4} by Auer, Cesa-Bianchi, Freund, and Schapire~\cite{AuerCBFS03-nonstochastic-bandit} and {\tt
  Exp4.P} by Beygelzimer, Langford, Li, Reyzin, and Schapire~\cite{BeygelzimerLLRS11-contextual-supervised-guarantees},
have regret bound of $O(\sqrt{KT\log N})=O(\sqrt{KT\cdot f\log n})$
(in expectation for~\cite{AuerCBFS03-nonstochastic-bandit} and with
high probability
in~\cite{BeygelzimerLLRS11-contextual-supervised-guarantees}).

Since our more general result (in Theorem~\ref{thm:main1}) has a
regret bound that depends on $T^{2/3}$, it is likely to be suboptimal.
However, when we assume that that the best arm is much better than the
other arms (in Section~\ref{sect:easy}), our regret bound is of the
right order, i.e., $\tilde{O}(\sqrt{KT\cdot\CSIZE})$.

While the adversarial bandit algorithms {\tt
  Exp4}~\cite{AuerCBFS03-nonstochastic-bandit} and {\tt
  Exp4.P}~\cite{BeygelzimerLLRS11-contextual-supervised-guarantees},
in the general case, perform better than the one presented here, these
algorithms have to process weight vector of size $N=\Omega(n^\CSIZE)$,
which can be very large.  Beygelzimer~{\em et al}~\cite{BeygelzimerLLRS11-contextual-supervised-guarantees} also
present algorithm {\tt VE} that competes against policies with finite
VC dimension $d$ in an i.i.d case.  This algorithm, in our case, is
again inefficient as it has to run against a policy set of size
$\Omega(\sqrt{T}^{\CSIZE})$, because the VC dimension is at least
$\CSIZE$.

More recent works such as Langford and
Zhang~\cite{LangfordZ07-epoch-greedy}, Dudik~{\em et
  al}~\cite{DudikHKKLRZ11-efficient-optimal}, and Agarwal~{\em et
  al}~\cite{AgarwalHKLLS14-taming} aim to reduce the dependency on $N$
on the running time and memory requirements.  Instead of keeping a
large weight vector, they make calls to a hypothesis searching oracle.
Agarwal~{\em et al}~\cite{AgarwalHKLLS14-taming} achieves the optimal
regret bound of $O(\sqrt{T\log(|\Pi|)})$, while making only
$O(\sqrt{T/\log(|\Pi|)})$ calls.  The drawback is that they all work
under the i.i.d. assumption, i.e., they assume that the context and
reward for each round is independently sampled from a fixed
distribution.  While our work assumes the stochastic model, i.e., the
reward is independently sampled for a given context, the contexts
given to the algorithm can be generated by an adversary.

\subsection{Comparison to other clustering results for bandit problems}

There are results that consider online clustering of linear bandits,
e.g.,~\cite{CBGZ-nips13-gang-bandits,
  pmlr-v32-gentile14-online-clustering,
  pmlr-v70-gentile17a-context-dependent,
  ijcai2019-405-improved-online-clustering-bandits}.
Cesa-Bianchi, Gentile, and Zappella~\cite{CBGZ-nips13-gang-bandits} consider a linear bandit problem with
social relationship modeled as an undirected graph $G=(V,E)$ where $V$
represents a set of $n$ users, each with an unknown parameter vector
$u_i\in\RR^d$.  The graph provides structures to the parameters $u_i$,
i.e., they assume that $\sum_{(i,j)\in E}\|u_i-u_j\|^2$ are small
compared to $\sum_{i\in V}\|u_i\|^2$.  The learning proceeds in
rounds.  For each round $t$, the user index $i_t$ and a set of
arbitrary context vectors $C_{i_t}=\{x_{t,1},\ldots,x_{t,c_t}\}$ is
presented and the learner has to pick one action $\bar{x_t}\in
C_{i_t}$ and receives a reward of $u_i^T \bar{x_t}$ with an additional
sub-Gaussian noise. Cesa-Bianchi~{\em et al}~\cite{CBGZ-nips13-gang-bandits} maintain a set of
$n$ linear bandit algorithms and an inverse correlation matrix $M_t$
for feedback sharing between bandit algorithms.  They obtain a regret
bound that depends on $\sqrt{nT}$ and log determinant of the matrix
$M_t$, which can be
$O(n)$. Gentile, Li, and Zappella~\cite{pmlr-v32-gentile14-online-clustering} consider a more
structured setting, where users can be partitioned into $m$ unknown
clusters and the context vectors $C_t$ in each round $t$ are generated
i.i.d. (where the size can be arbitrary).  We note that this setting
is closely related to our work where $m=\CSIZE+1$.
Gentile~{\em et al}~\cite{pmlr-v32-gentile14-online-clustering} give a regret bound that
depends on $\sqrt{mT}$ with additional $O(n+m)$ terms that are
constant with $T$.  A recent
result by Gentile~{\em et al}~\cite{pmlr-v70-gentile17a-context-dependent} considers various
data-dependent assumptions to obtain sharper bounds that depend on
$\sqrt{Tm}$ and $n\cdot\mathrm{polylog}(nT)$.

Another line of work, by Maillard and Munos\cite{pmlr-v32-maillard14-latent-bandits} and Hong~{\em et al}~\cite{NEURIPS2020_9b7c8d13-latent-bandits-revisited},
considers latent bandits where there is a
partition of context types ${\mathcal B}$ into $C$ clusters $\mathcal
C=\{{\mathcal B}_c\}$, each with known reward distribution.  However,
the learner, when receiving the context type $b\in{\mathcal B}$, does
not know the cluster ${\mathcal B}_c$ containing $b$.  This is a much
harder problem.  Our setting provides more structures, in forms of
graphs, that guides the clustering of vertices.  We, however, do not
know the reward distributions before hand.

\subsection{Other related works}

There are other works on contextual bandit problems with some
structure.  See, for example,~\cite{Agrawal:1995:CBP:218209.218238,
  Kleinberg:2008:MBM:1374376.1374475},
and~\cite{Slivkins14-Contextual-Similarity} that consider similarity
between contexts and arms.

There are numerous extensions to the graph label prediction problems,
see, e.g.,~\cite{HerbsterPW05-icml-learning-graphs,
  HerbsterP06-nips-perceptron, HerbsterLP08-large-diameter,
  Cesa-BianchiGV09-optimal-tree, GentileHP13},
and~\cite{PasterisVGH18-similarity-pairwise}.

Alon~{\em et al}~\cite{AlonCGMMS17-graph-feedback} consider a multi-armed bandit
problem when the feedback is constrained with a feedback graph.


\section{Acknowledgements}

We would like to thanks reviewers from COLT'19, ALT'22, and IPL for
many useful comments.  Especially, reviewers from ALT'22 who suggested
a simpler and better algorithm presented here.  Both authors are
supported by the Thailand Research Fund~[Grant number RSA-6180074].

\bibliographystyle{elsarticle-num}
\bibliography{graphcontexts}


\end{document}